\documentclass[11pt]{article}

\usepackage{epsfig,epsf,fancybox}
\usepackage{amsmath}
\usepackage{mathrsfs}
\usepackage{amssymb}
\usepackage{graphicx}
\usepackage{color}
\usepackage{multirow}
\usepackage{paralist}
\usepackage{verbatim}
\usepackage{galois}
\usepackage{algorithm}
\usepackage{algorithmic}
\usepackage{boxedminipage}
\usepackage{booktabs}
\usepackage{stmaryrd}
\usepackage{subfig}
\usepackage{wrapfig}
\usepackage[bottom]{footmisc}

\usepackage{hyperref}

\usepackage{nicefrac}
\usepackage[left=1in,top=1in,right=1in,bottom=1in,letterpaper]{geometry}

\numberwithin{equation}{section}

\providecommand{\keywords}[1]{\textbf{\textit{Keywords---}} #1}

\newcommand{\norm}[1]{\left\lVert#1\right\rVert}

\newtheorem{theorem}{Theorem}[section]
\newtheorem{corollary}[theorem]{Corollary}
\newtheorem{lemma}[theorem]{Lemma}

\newtheorem{definition}[theorem]{Definition}

\newtheorem{assumption}[theorem]{Assumption}
\newtheorem{remark}[theorem]{Remark}

\newcommand{\M}{\mathcal{M}}
\newcommand{\U}{\mathcal{U}}

\newcommand{\cM}{\mathcal{M}}

\newcommand{\etal}{ et al. }
\newcommand{\argmin}{\mathop{\rm argmin}}
\newcommand{\RR}{\mathbb{R}}
\newcommand{\BB}{\mathbb{B}}
\newcommand{\PP}{\mathbb{P}}

\newcommand{\be}{\begin{equation}}
\newcommand{\ee}{\end{equation}}
\newcommand{\ba}{\begin{array}}
\newcommand{\ea}{\end{array}}
\newcommand{\bad}{\begin{aligned}}
\newcommand{\ead}{\end{aligned}}
\newcommand{\prox}{\mathrm{prox}}

\newcommand{\ufun}{\mathscr{F}}
\newcommand{\uspace}{\mathscr{S}}
\newcommand{\utime}{\mathscr{T}}
\begin{document}

\title{Escaping Saddle Points for Nonsmooth Weakly Convex Functions via Perturbed Proximal Algorithms}

\author{Minhui Huang \thanks{Department of Electrical and Computer Engineering, University of California, Davis} \and Weiming Zhu\thanks{Simon Business School, University of Rochester}}
\date{\today}
\maketitle

\begin{abstract}
We propose perturbed proximal algorithms that can provably escape strict saddles for nonsmooth weakly convex functions. The main results are based on a novel characterization of $\epsilon$-approximate local minimum for nonsmooth functions, and recent developments on perturbed gradient methods for escaping saddle points for smooth problems. Specifically, we show that under standard assumptions, the perturbed proximal point, perturbed proximal gradient and perturbed proximal linear algorithms find $\epsilon$-approximate local minimum for nonsmooth weakly convex functions in $O(\epsilon^{-2}\log(d)^4)$ iterations, where $d$ is the dimension of the problem.
\end{abstract}


\section{Introduction}

Nonconvex optimization plays an important role in deep learning and machine learning. Although global optimum for nonconvex optimization is not easy to obtain in general, recent studies showed that for many smooth optimization problems arising from important applications, second-order stationary points are indeed global optimum under mild conditions. %These applications include 
%\blue{tensor decomposition \cite{ge2015escaping}, dictionary learning \cite{sun2016complete}, phase retrieval \cite{sun2018geometric}, synchronization
%and MaxCut \cite{bandeira2016low, mei2017solving}, smooth semidefinite programs \cite{boumal2016non}, matrix sensing \cite{bhojanapalli2016global}, matrix completion \cite{ge2016matrix} and robust principal component analysis \cite{ge2017no} etc.} 
We refer to \cite{ge2015escaping,sun2016complete,sun2018geometric,bandeira2016low, mei2017solving,boumal2016non,bhojanapalli2016global,ge2016matrix,ge2017no} for a partial list of these works.
There have been many research works on this topic when the objective function is smooth. As a result, in order to find the global optimum for these problems, one only needs to avoid the saddle points. {Besides these highly structured optimization problems, avoiding saddle points is also crucial for general optimization problems. For instance, in training of neural network, recent results show that the proliferation of saddle points undermines the quality of solutions, which lead to much higher error than local minimum \cite{dauphin2014identifying, choromanska2015loss}. Neural network saddle points appear even more frequently and are potentially a major bottleneck in high dimensional problems of practical interest \cite{dauphin2014identifying}. On the other hand, nonsmooth formulations frequently appear in machine learning and modern signal processing problems. Some well studied examples include neural networks with ReLU activation \cite{li2017convergence}, robust principal component analysis \cite{yi2016fast, gu2016low}, robust matrix completion \cite{li2020nonconvex, cambier2016robust}, blind deconvolution \cite{lau2019short, zhang2017global}, robust phase retrieval \cite{duchi2019solving, davis2017nonsmooth}. It is known that neural networks with the nonsmooth ReLU activation has great expressive power \cite{li2017convergence}. Nonsmooth regularizers also provide robustness against sparse outliers for data science applications.} In this paper, we consider how to escape saddle points for nonsmooth functions, which is still less developed so far.

Most existing works for {finding local minimum} consider smooth objective functions {and derive the complexity for reaching an $\epsilon$-second order stationary point (See Definition \ref{defn:nesterov-eps-2nd-sp}). Nesterov and Polyak \cite{nesterov2006cubic} proposed the cubic regularization method, which requires $O(\epsilon^{-3/2})$ iterations for obtaining an $\epsilon$-second order stationary point. Curtis \etal \cite{curtis2017trust} proved the same complexity result for the trust region method. To avoid Hessian computation required in \cite{nesterov2006cubic,curtis2017trust}, Carmon \etal \cite{carmon2018accelerated} and Agarwal \etal \cite{agarwal2017finding} proposed to use Hessian-vector product and achieved convergence rate of $O(\epsilon^{-7/4})$. Recently, the complexity results of pure first-order methods for obtaining local minimum have been studied (see, e.g., \cite{ge2015escaping,daneshmand2018escaping,jin2019nonconvex,fang2019sharp}). 
Lee \etal \cite{lee2016gradient} proved that gradient descent method converges to a local minimizer almost surely with random initialization by using tools from dynamical systems theory. However, Du \etal \cite{du2017gradient} showed that gradient descent (GD) method may take exponential time to escape saddle points. Recently, Jin \etal \cite{jin2017escape,jin2019nonconvex} proved that the perturbed GD can converge to a local minimizer in a number of iterations that depends poly-logarithmically on the dimension, {reaching an iteration complexity of $\tilde{O}(\epsilon^{-2}\log (d)^4)$}. Jin \etal \cite{jin2019nonconvex} also proposed perturbed stochastic gradient descent method (SGD) that requires $\tilde{O}(\epsilon^{-4})$ stochastic gradient computations to reach an $\epsilon$-second order stationary point. {By utilizing the techniques developed by Carman \etal \cite{carmon2017convex}, Jin \etal \cite{jin2017accelerated} proposed the accelerated perturbed GD and improved the complexity to $\tilde{O}(\epsilon^{-7/4}\log (d)^6)$.}
%has an iteration complexity of $\tilde{O}(\text{poly}(d)\epsilon^{-8})$. Daneshmand \etal \cite{daneshmand2018escaping} improved the convergence rate of stochastic gradient method to $\tilde{O}(\epsilon^{-5})$. 
More recently, Fang \etal \cite{fang2019sharp} proved that the complexity of SGD can be further improved to $\tilde{O}(\epsilon^{-3.5})$ under the assumption that the gradient and Hessian are both Lipschitz continuous. %When it comes to the adaptive gradient method, Staib \etal \cite{staib2019escaping} proved that equipping stochastic gradient method with a preconditioner can escape saddle points even faster.
How to escape saddle points for constrained problems are also studied in the literature. Avdiukhin \etal \cite{avdiukhinescaping} proposed a noisy projected gradient descent method for escaping saddle points for problems with linear inequality constraints. {Mokhtari \etal \cite{mokhtari2018escaping} studied a generic algorithm for escaping saddle points of smooth nonconvex optimization problems subject to a general convex set.} Criscitiello and Boumal \cite{criscitiello2019efficiently} and Sun \etal \cite{sun2019escaping} considered the perturbed Riemannian gradient method and showed that it can escape saddle points for smooth minimization over manifolds.

The literature on escaping saddle points for nonsmooth functions, on the other hand, is relatively limited. Among few existing works, Huang and Becker \cite{huang2019perturbed} proposed a perturbed proximal gradient method for nonconvex minimization with an $\ell_1$-regularizer and showed that their proposed method can escape the saddle points for this particular class of problems. It is not clear how to extend the results in \cite{huang2019perturbed} to more general nonsmooth functions. Davis and  Drusvyatskiy \cite{davis2019active} extended the work \cite{lee2016gradient} to nonsmooth problems, showing that proximal point, proximal gradient and proximal linear algorithms converge to local minimum almost surely, under the assumption that the objective function satisfies a strict saddle property. However, no convergence rate complexity {for finding a local minimum} was given in \cite{davis2019active}.

In this paper, we consider minimizing a nonsmooth weakly convex function. We propose perturbed proximal point, perturbed proximal gradient and perturbed proximal linear algorithms, and prove that they can escape active strict saddle points for nonsmooth weakly convex functions, under standard conditions. Our main results are based on a novel characterization on the $\epsilon$-approximate local minimum inspired by \cite{davis2019active}, and the analysis of perturbed gradient descent method for minimizing smooth functions \cite{jin2017escape, jin2019nonconvex}. Our {\bf main contributions} are summarized below.
\begin{enumerate}
\item We propose a novel definition of $\epsilon$-approximate local minimum for nonsmooth weakly convex problems.% under reasonable assumptions.
\item We propose three perturbed proximal algorithms that provably escape active strict saddle points for nonsmooth weakly convex functions. The three algorithms are: perturbed proximal point algorithm, perturbed proximal gradient algorithm, and perturbed proximal linear algorithm.
\item We analyze the iteration complexity of the three proposed algorithms for obtaining an $\epsilon$-approximate local minimum for nonsmooth weakly convex functions. We show that the iteration complexity for the three algorithms to obtain an $\epsilon$-approximate local minimum is $O(\epsilon^{-2}\log (d)^4)$. To the best of our knowledge, this is the first quantitative iteration complexity of algorithms for finding an approximate local minimum for nonsmooth weakly convex functions.
\end{enumerate}

{\textbf{Notation:} Let $\norm{\cdot}$ denote the Euclidean norm of a vector. For a nonsmooth function $f(\cdot)$, we denote its subdifferential at point $x$ as $\partial f(x)$, and the parabolic subderivative (defined in the appendix) at $x$ for $u \in \text{dom }df(x)$ with respect to $w$ as $df^2(x)(u|w)$. When $df^2(x)(u|w)$ is a constant with respect to $w$, we omit $w$ and denote it as $df^2(x)(u)$.  We denote $\M$ as a manifold and $\U$ as a neighborhood set of some points. The tangent space of a manifold $\M$ at $x$ is denoted as $T_\M(x)$. We further denote a restriction $f_\M(x):= f + \delta_\M$, where $\delta_\M$ evaluates to $0$ on $\M$ and $+\infty$ off it. We denote the operator of three proximal algorithms as $S(\cdot)$ and its Jacobian as $\nabla S(\cdot)$. A ball centered at $x$ with radius $r$ is denoted as $\mathbb{B}_x(r)$. We use big-$O$ notation, where $p = O(q)$ if there exists a global constant $c$ such that $|p| \le cq$ and $p = \tilde{O}(q)$ that hides a poly-logarithmic factor of $d$ and $\epsilon$. We use $\lambda_{\min}(Z)$  and $\lambda_{\max}(Z)$ to denote the smallest and largest eigenvalues of matrix $Z$, respectively.
}

\section{Perturbed Proximal Algorithms}\label{sec:prox-algs}
We consider three types of nonsmooth and weakly convex optimization problems and their corresponding solution algorithms: proximal point algorithm (PPA), proximal gradient method (PGM), and proximal linear method (PLM). Specifically, PPA solves
\be\label{ppa-prob}\min_{x\in \RR^d} f(x), \ee
and it iterates as
\be\label{ppa-alg} x_{t+1} := S(x_t), \quad \mbox{where} \quad S(x):= \prox_{\eta f}(x).\ee
%\textcolor{red}{should this be $\prox_{\lambda f}(x)$ here?}
Here we assume that $f: \RR^d \to \RR\cup\{\infty\} $ is a nonsmooth $\ell$-weakly convex function. %p, and \blue{$\lambda \in (0, \ell^{-1})$(better use $\lambda$ for the parameter of the moreau envelop and $\eta$ for the step size only)}.
Recall that a function $f(x)$ is $\ell$-weakly convex if $f(x) + \frac{\ell}{2}\norm{x}^2$ is convex. %The proximal mapping of an $\alpha$-weakly convex function $w$ is defined as ($\alpha \leq 1$):
%\[\prox_w(x) := \argmin_y \ w(y) + \frac{1}{2}\|y-x\|_2^2.\]
When choosing {$\lambda \in (0, \ell^{-1})$}, we define the Moreau envelope of function $f(\cdot)$, denoted as {$f_\lambda(\cdot)$}, and its corresponding proximal mapping as {
\[
f_\lambda(x) = \min_{y\in \RR^d} f(y) +  \frac{1}{2\lambda}\norm{y-x}^2, \quad\text{and} \quad \prox_{\lambda f}(x) = \argmin_{y\in \RR^d} f(y) +  \frac{1}{2\lambda}\norm{y-x}^2.
\]
}
{Note that the step size $\eta$ in \eqref{ppa-alg} and the parameter $\lambda$ can be different.}

PGM solves
\be\label{pgm-prob} \min_{x\in \RR^d} f(x) \equiv g(x) + {m}(x),\ee
and it iterates as
\be\label{pgm-alg} x_{t+1} := S(x_t), \quad \mbox{where} \quad S(x) := \prox_{\eta {m}}(x - \eta\nabla g(x)).\ee
Here we assume $g: \RR^d \to \RR\cup\{\infty\}$ is $C^2$-smooth with $\beta$-Lipschitz gradient and ${m}: \RR^d \to \RR\cup\{\infty\}$ is closed and $\mu$-weakly convex.

PLM solves
\be\label{plm-prob} \min_{x\in \RR^d} f(x) \equiv h(F(x)) + {m}(x),\ee
and it iterates as
\be\label{plm-alg} x_{t+1} := S(x_t), \quad \mbox{where} \quad S(x):= \argmin_{y\in \RR^d} h(F(x) + \nabla F(x)(y-x)) + {m}(y) + \frac{1}{2\eta}\norm{y-x}^2.\ee
Here we assume $F: \RR^d \to \RR^m$ is a $C^2$-smooth map, $h:\RR^d \to \RR\cup\{\infty\}$ is nonsmooth convex, and ${m}: \RR^d \to \RR\cup\{\infty\}$ is $\mu$-weakly convex.

Now we present a unified algorithmic framework of perturbed proximal algorithms that find $\epsilon$-approximate local minimum points (will be defined later) of \eqref{ppa-prob}, \eqref{pgm-prob} and \eqref{plm-prob}. Our algorithmic framework is presented in Algorithm \ref{algo:PPA} and it largely follows the perturbed gradient method for smooth problems \cite{jin2017escape,jin2019nonconvex}. Inspired by \cite{davis2019stochastic}, Algorithm \ref{algo:PPA} utilizes the Moreau envelope to measure the first order stationarity. {When $\norm{\nabla f_\lambda(x)}$ is close to zero, then $x$ is close to a first order stationary point. }Same as \cite{jin2017escape,jin2019nonconvex}, the perturbation added to $x_t$ helps escape the saddle points. In Algorithm \ref{algo:PPA}, $\xi\sim\text{Uniform}(\BB_0(r))$ means that $\xi$ is a random vector uniformly sampled from the Euclidean ball with radius $r$. When $S$ is chosen as in \eqref{ppa-alg}, \eqref{pgm-alg} and \eqref{plm-alg}, we name Algorithm \ref{algo:PPA} as {\bf perturbed PPA}, {\bf perturbed PGM} and {\bf perturbed PLM}, respectively.

%In this section, we present a unified framework of perturbed proximal algorithms, which utilizes the perturbation technique \cite{jin2019nonconvex} to find a local minimum of nonsmooth functions.The proximal operator $S_\eta(x)$ can be one of the proximal point, proximal gradient and proximal linear operators defined in Table \ref{table:update_map}. When $\norm{\nabla f_\lambda(x)} \le \epsilon$ is satisfied, we have approached an approximate critical point. If we haven't add random perturbations in the last $\utime$ iterations, we generalize $\xi_t$ uniformly on the sphere of a ball $B_0(r)$ as a random perturbation vector and add it to the current $x$. Then we keep applying the proximal operator $S_\eta(x)$ in the following $\utime$ iterations.

\begin{algorithm}[H]
\caption{Perturbed Proximal Algorithms}\label{algo:PPA}
\begin{algorithmic}
\renewcommand{\algorithmicrequire}{\textbf{Input: }}
\renewcommand{\algorithmicensure}{\textbf{Output: }}
\REQUIRE $x_0$, step size $\eta$, perturbation radius $r$, time interval $\utime$, tolerance $\epsilon$.
\STATE $t_{\text{perturb}} = 0$
\FOR{$t = 0, 1, \ldots, T $}
\IF{$\norm{\nabla f_\lambda(x_t)} \le \epsilon$ and $t - t_{\text{perturb}} > \utime$}
\STATE $x_t \leftarrow x_t - \eta\xi_t, ~(\xi_t \sim \text{Uniform}(\BB_0(r))); \quad t_{\text{perturb}} \leftarrow t$
\ENDIF
\STATE $x_{t+1} \leftarrow S(x_t)$, where $S$ is one of the operators defined in \eqref{ppa-alg}, \eqref{pgm-alg} and \eqref{plm-alg}.
\ENDFOR
\end{algorithmic}
\end{algorithm}
%\textcolor{red}{Explain the algorithm little bit, especially the conditions under which we add perturbation.}

\begin{remark}

Calculating $\norm{\nabla f_\lambda(x_t)}$ in each iteration can be time consuming in the worst case. However, it is investigated by Davis \etal \cite{davis2019stochastic} that $\norm{\nabla f_\lambda(x_t)}$ is proportional to more familiar quantities such as the gradient mapping: $\lambda^{-1}\norm{x_t - x_{t-1}}$. In practice, we can replace the condition $\norm{\nabla f_\lambda(x_t)} < \epsilon$ by $\lambda^{-1}\norm{x_t - x_{t-1}}\le \epsilon$, which is much easier to check.
\end{remark}

%, finding an approximate critical point. After approaching one of the critical points, we add a uniformly random perturbation vector to rule out active saddle points. We will show that if the current point is a strict saddle, then with high probability adding perturbations will make sufficient decrease in the next $\utime$ iterations for the Moreau Envelop. We then specify some details for different proximal operators.

\section{Active Strict Saddles and $\epsilon$-Approximate Local Minimum}

We first introduce the concepts of active manifold and strict saddles, which play an important role in characterizing nonsmooth functions.
\begin{definition}[Active manifold \cite{davis2019active}]\label{defn:ident_man}
Consider a closed weakly convex function $f\colon\RR^d\to\RR\cup\{\infty\}$ and fix a set $\mathcal{M} \subseteq \RR^d$ containing a critical point $x$ of $f$. Then $\mathcal{M}$ is called {\em an active} $C^p${\em-manifold around} $x$ if there exists a neighborhood $\U$ around $x$ satisfying the following conditions.
\begin{enumerate}
\item {Smoothness.} The set $\mathcal{M}\cap \U$ is a $C^p$-smooth manifold and the restriction of $f$ to $\mathcal{M}\cap \U$ is $C^p$-smooth.
\item {Sharpness.} The lower bound holds:
\[\inf \{\|v\|: v\in \partial f(x),~x\in \U\setminus \cM\}>0.\]
\end{enumerate}
\end{definition}

\begin{definition}[Strict Saddles \cite{davis2019active}]\label{defn:strict_saddle}
Consider a weakly convex function $f: \RR^d \to \RR\cup\{\infty\}$. A critical point $x$ is a strict saddle of $f$ if there exists a $C^2$-active manifold $\M$ of $f$ at $x$ and the inequality $d^2f_\M(x)(u) < 0$ holds for some vector $u \in T_\M(x)$. A function $f$ is said to have the strict saddle property if each of its critical points is either a local minimizer or a strict saddle.
\end{definition}

We now provide some intuitions on the geometric structure of the nonsmooth optimization. The nonsmoothness of optimization problems typically arises in a highly structured way, usually along a smooth active manifold. Geometrically, an active manifold $\M$ is a set that contains all points such that the objective function is smooth along the manifold, but varies sharply off it. For example, Figure \ref{fig:function_sur} illustrates the surface of function $g(x,y) = \lvert x ^2 + y^2 - 1 \rvert + x$ and its active manifold is $\M = \{(x, y) | x ^2 + y^2 = 1\}$. There is a saddle point $(1, 0)$ and a global minimum $(-1, 0)$ of function $g$ and both of them lie on the active manifold $\M$. The active manifold idea is widely used in many problems, for instance, the point-wise maximum of smooth function problem and the maximum eigenvalue minimization problem \cite{lewis2002active}. In \cite{hare2007identifying}, the authors showed that some simple algorithms such as projected gradient and proximal point methods can identify active manifold, which means all iterates after a finite number of iterations must lie on the active manifold. In this paper, we assume all critical points of the objective function lie on the active manifold (otherwise it becomes a smooth problem in the neighborhood of the critical point) and the active manifold is $C^2$ smooth in the neighborhood of each critical point. % We would argue that such an assumption is generic for common nonsmooth functions, for example, the $l_1$-regularizer.\\

\begin{figure}[h!]
  \centering
  \includegraphics[scale=0.18]{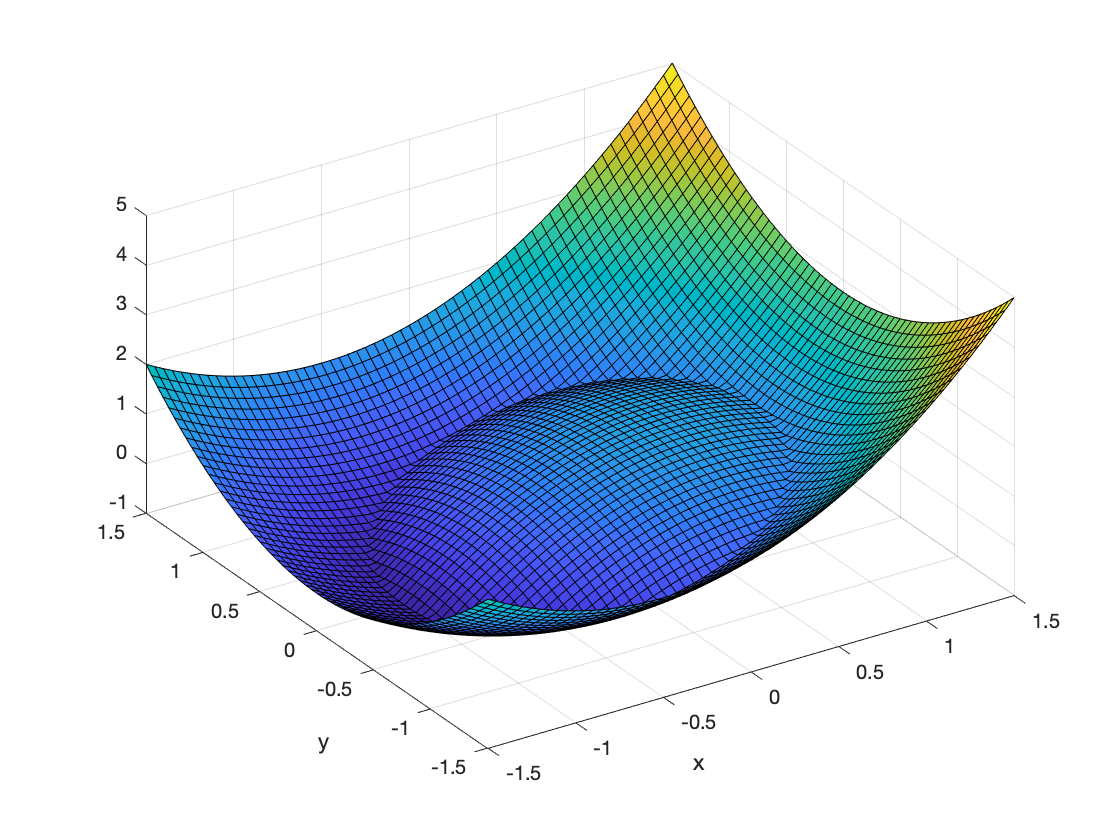}
  \caption{Surface of the function $g(x,y) = \lvert x ^2 + y^2 - 1 \rvert+ x.$}
  \label{fig:function_sur}
\end{figure}

Throughout this paper, we make the following assumption.
\begin{assumption}\label{assum:strict-saddle-property}
The function $f$ in \eqref{ppa-prob}, \eqref{pgm-prob} and \eqref{plm-prob} satisfies the strict saddle property.
\end{assumption}
From Assumption \ref{assum:strict-saddle-property} we know that any algorithms for solving \eqref{ppa-prob}, \eqref{pgm-prob} and \eqref{plm-prob} converge to a local minimum if they can escape the strict saddle points.

We now discuss how to define an approximate local minimum for \eqref{ppa-prob}, \eqref{pgm-prob} and \eqref{plm-prob}. For nonconvex smooth problems, Nesterov and Polyak \cite{nesterov2006cubic} proposed the following definition for an $\epsilon$-second-order stationary point.
\begin{definition}[$\epsilon$-second-order stationary point for smooth problem \cite{nesterov2006cubic}]\label{defn:nesterov-eps-2nd-sp}
Assume function $f\in C^2$, and the Hessian of $f$ is Lipchitz continuous with a Lipschitz constant $\rho$. A point $x$ is called an $\epsilon$-second-order stationary point to problem $\{\min_x f(x)\}$, if it satisfies $\norm{\nabla f(x)} \le \epsilon$ and $\lambda_{\min}(\nabla^2 f(x)) > -\sqrt{\rho \epsilon}.$ %\blue{Delete:  , where $\lambda_{\min}(Z)$ denotes the smallest eigenvalue of matrix $Z$}
\end{definition}
This definition does not apply to our problems because we consider nonsmooth functions. In the following, we propose a novel definition of $\epsilon$-approximate local minimum for nonsmooth functions. Our definition is motivated by a recent result of Davis and Drusvyatskiy \cite{davis2019active}, which proved the following results regarding the proximal algorithms PPA \eqref{plm-alg}, PGM \eqref{pgm-alg} and PLM \eqref{plm-alg} discussed in Section \ref{sec:prox-algs}.

\begin{theorem}[Theorems 3.1, 4.1 and 5.3 in \cite{davis2019active}]\label{thm:Davis-3-theorems}
Assume the objective function $f$ in problems \eqref{ppa-prob}, \eqref{pgm-prob} and \eqref{plm-prob} is an $\ell$-weakly convex and $\bar{x}$ is a critical point of $f$. For parameters $\lambda \in (0, \ell^{-1})$ and sufficiently small $\eta>0$, the following statements hold.
\begin{enumerate}
\item Consider problem \eqref{ppa-prob}. Suppose that $f$ admits a $C^2$ active manifold $\M$ at $\bar{x}$. Then the proximal map $S(\cdot):=\prox_{\eta f}(\cdot)$ defined in \eqref{ppa-alg} is $C^1$-smooth on a neighborhood of $\bar{x}$. Moreover, if $\bar{x}$ is a strict saddle point of $f$, then $\bar{x}$ is both a strict saddle point of $f_\lambda$ and an unstable fixed point of the proximal map $S$. Moreover, $\nabla S(\bar{x})$ has a real eigenvalue that is strictly greater than one.
\item Consider problem \eqref{pgm-prob}. Suppose that $f$ admits a $C^2$ active manifold $\M$ at $\bar{x}$. Then the proximal-gradient map $S$ defined in \eqref{pgm-alg} is $C^1$-smooth on a neighborhood of $\bar{x}$. Moreover, if $\bar{x}$ is a strict saddle point of $f$, then $\nabla S(\bar{x})$ has a real eigenvalue that is strictly greater than one.
\item Consider problem \eqref{plm-prob}. Suppose the problem admits a composite $C^2$ active manifold $\M$ \cite[Definition 5.1]{davis2019active} at $\bar{x}$. Then the proximal linear map $S$ defined in \eqref{plm-alg} is $C^1$-smooth on a neighborhood of $\bar{x}$. Moreover, if $\bar{x}$ is a composite strict saddle point \cite[Definition 5.1]{davis2019active}, then the Jacobian $\nabla S(\bar{x})$ has a real eigenvalue strictly greater than one.
\end{enumerate}
\end{theorem}

Now we recall some properties of Moreau envelope. 
\begin{lemma}[Basic Properties of Moreau Envelope]\label{lem_prop_moreau}
Consider an $\ell$-weakly convex function $f\colon\RR^d\to\RR\cup\{\infty\}$ and fix a parameter $\lambda\in(0,\ell^{-1})$. The following statements are true.
\begin{enumerate}
\item The envelope $f_{\lambda}(\cdot)$ is $C^1$-smooth with its gradient given by
\begin{equation}\label{eqn:grad_moreau}
\nabla f_{\lambda}(x)=\lambda^{-1}(x-\prox_{\lambda f}(x)).
\end{equation}
\item The proximal map $\prox_{\lambda f}(\cdot)$ is $\frac{1}{1-\lambda \ell}$-Lipschitz continuous and the gradient map $\nabla f_{\lambda}$ is Lipschitz continuous with constant $\max\{\lambda^{-1},\frac{\ell}{1-\lambda \ell}\}$.
\item The critical points of $f(\cdot)$ and $f_{\lambda}(\cdot)$ coincide. In particular, they are exactly the fixed points of the proximal mapping $\prox_{\lambda f}(\cdot)$. Moreover, if $\bar{x}$ is a local minimum of $f_{\lambda}(\cdot)$, then $\bar{x}$ is a local minimum of $f(\cdot)$.
\end{enumerate}
\end{lemma}

Theorem \ref{thm:Davis-3-theorems} and Lemma \ref{lem_prop_moreau} immediately implies the following sufficient condition for local minimum of problems \eqref{ppa-prob}, \eqref{pgm-prob} and \eqref{plm-prob}.
\begin{corollary}\label{local_min}
Assume Assumption \ref{assum:strict-saddle-property} holds. For problems \eqref{ppa-prob} and \eqref{pgm-prob}, assume function $f(x)$ is nonsmooth $\ell$-weakly convex and admits a $C^2$ active manifold at $\bar{x}$. For problem \eqref{plm-prob}, assume function $f(x)$ is nonsmooth $\ell$-weakly convex and admits a $C^2$ composite active manifold at $\bar{x}$. Then $\bar{x}$ is a local minimum of \eqref{ppa-prob}, \eqref{pgm-prob}, and \eqref{plm-prob} if the following holds for $\lambda\in(0,\ell^{-1})$:
\[\norm{\nabla f_\lambda(x)} = 0, \quad \text{and} \quad \lambda_{\max}(\nabla S(x)) \le 1.\]
\end{corollary}

From Corollary \ref{local_min}, we can naturally define $\epsilon$-approximate local minimum for problems \eqref{ppa-prob}, \eqref{pgm-prob} and \eqref{plm-prob} as follows.
\begin{definition}[$\epsilon$-approximate local minimum]\label{defn:localmin_general}
Assume Assumption \ref{assum:strict-saddle-property} holds. For problems \eqref{ppa-prob} and \eqref{pgm-prob}, assume function $f$ is nonsmooth $\ell$-weakly convex and admits a $C^2$ active manifold at $\bar{x}$. For problem \eqref{plm-prob}, assume function $f$ is nonsmooth $\ell$-weakly convex and admits a $C^2$ composite active manifold at $\bar{x}$. Assume $\nabla S(\cdot)$ is Lipschitz continuous in the neighborhood of each critical point with Lipschitz constant $\eta\rho$. We call $\bar{x}$ an $\epsilon$-approximate local minimum of problems \eqref{ppa-prob}, \eqref{pgm-prob} and \eqref{plm-prob} if the following holds for $\lambda\in(0,\ell^{-1})$:
\be\label{defn:localmin_general-inequality}\norm{\nabla f_\lambda(\bar{x})} \le \epsilon, \quad \lambda_{\max}(\nabla S(\bar{x})) < 1 + \eta\sqrt{\rho\epsilon}.\ee
\end{definition}

\begin{remark}
Here we only assume $\nabla S(\cdot)$ is Lipschitz continuous in the neighborhoods of all critical points. % $S_\eta(\cdot)$ is $C^1$ in the neighborhood of each critical point. $S_\eta(\cdot)$ can be nonsmooth globally on some unfixed points. We claim that this local property is enough in our later analysis.
The $\epsilon$-approximate local minimum generalizes the concept of $\epsilon$-second-order stationary point \cite{nesterov2006cubic} to nonsmooth problems. For smooth problems, the gradient descent operator is defined as $S(x) = x - \eta \nabla f(x)$, whose Jacobian is $\nabla S(x) = I - \eta \nabla^2 f(x)$. By replacing $\norm{\nabla f_\lambda(x)} \le \epsilon$ in \eqref{defn:localmin_general-inequality} by $\norm{\nabla f(x)} \le \epsilon$, the conditions in \eqref{defn:localmin_general-inequality} are the same as the ones required in the definition of $\epsilon$-second-order stationary point (Definition \ref{defn:nesterov-eps-2nd-sp}) for smooth problems.
\end{remark}

\section{Main Results}
In this section, we provide iteration complexity results of the three perturbed proximal algorithms in Algorithm \ref{algo:PPA} for obtaining an $\epsilon$-approximate local minimum. %We first focus on an easy case, the perturbed proximal point method. We then provide a unified framework for all three perturbed proximal algorithms and analyze their iteration complexity for reaching an $\epsilon$ approximate local minimum. A full proof of all results is provided in the appendix.

\subsection{An Easy Case: Perturbed Proximal Point Algorithm}
The perturbed PPA is easy to analyze, because its complexity can be obtained by directly applying the results in \cite{jin2019nonconvex}. To see this, %we first recall some properties of Moreau envelope.

note that Lemma \ref{lem_prop_moreau} immediately implies the following result.
\begin{corollary}\label{cor:moreau-gradient-descent}
PPA \eqref{ppa-alg} with $\eta=\lambda$ for solving \eqref{ppa-prob} is equivalent to
\[x_{t+1} := x_t -\lambda\nabla f_\lambda(x_t),\]
i.e., the gradient descent algorithm for
\be\label{min-f-lambda}\min \ f_\lambda(x), \ee
with step size $\lambda$.
\end{corollary}

Corollary \ref{cor:moreau-gradient-descent} indicates that the perturbed PPA for solving \eqref{ppa-prob} is exactly the same as the perturbed gradient method for solving \eqref{min-f-lambda}. Therefore, we can give the iteration complexity of perturbed PPA using the iteration complexity of the perturbed gradient method from \cite{jin2019nonconvex}.

\begin{theorem}\label{thm:proximalpoint}
Denote the optimal value of \eqref{min-f-lambda} as $f_\lambda^\star$. Assume function $f(\cdot)$ is nonsmooth $\ell$-weakly convex and admits a $C^2$-smooth manifold around all its critical points. This implies that $f_\lambda(\cdot)$ is twice differentiable and its Hessian is Lipschitz continuous around all its critical points. We use $\rho$ to denote the Lipschitz constant of the Hessian of $f_\lambda(\cdot)$.
For any $\epsilon \in (0,1/(\lambda^2\rho))$ and $\delta\in(0,1)$, we set the parameters as
\be\label{eq:para_pp}
\lambda < \frac{\ell^{-1}}{2}, \quad \eta = \lambda, \quad r =\frac{\epsilon}{400 \iota^3}, \quad \utime = \frac{1}{\lambda \sqrt{\rho\epsilon}} \cdot\iota, \quad \ufun = \frac{1}{50 \iota^3} \sqrt{\frac{\epsilon^3}{\rho}}, \quad \uspace = \frac{1}{4\iota} \sqrt{\frac{\epsilon}{\rho}},
\ee
where $\iota = c \cdot \log( d(f_\lambda(x_0) - f_\lambda^\star)/(\lambda\rho\epsilon\delta))$, and $c$ is a sufficiently large absolute constant. We set the number of iterations as
\[
T = 4 \max\left\{\frac{(f_\lambda(x_0) - f_\lambda^*)}{\ufun / (2\utime)}, \frac{(f_\lambda(x_0) - f_\lambda^*)}{ \eta\epsilon^2/2}\right\}  = O\left(\frac{(f_\lambda(x_0) - f_\lambda^*)}{\lambda\epsilon^2} \cdot (\log d)^4\right).
\]
Then the Perturbed PPA (Algorithm~\ref{algo:PPA} with $S(x) = \prox_{\eta f}(x)$) has the property that with probability at least $1-\delta$, at least one half of its iterates in the first $T$ iterations satisfy
\be\label{local_min_pp}
\norm{\nabla f_\lambda(x)} \le \epsilon, \quad \lambda_{\min}(\nabla^2 f_\lambda(x)) > -\sqrt{\rho \epsilon}.
\ee
%According to Lemma \ref{lem_prop_moreau}, this indicates that $x$ is an $\epsilon$-approximate local minimum of problem \eqref{ppa-prob}.
This indicates that $x$ is an $\epsilon$-approximate local minimum of problem \eqref{ppa-prob} according to Definition \ref{defn:localmin_general}.
\end{theorem}
%By lemma \ref{lem_prop_moreau}, the critical points of the Moreau Envelop $f_\lambda(\cdot)$ and the function $f(\cdot)$ coincide. Therefore, the approximate local minimum of the Moreau Envelop is also the approximate local minimum of function $f(\cdot)$. Also, we observe that condition \eqref{local_min_pp}  is a special case of the conditions in definition \ref{defn:localmin_general}.

\subsection{Iteration Complexity of the Three Perturbed Proximal Algorithms}
In this section, we provide a unified analysis of the iteration complexity of the three perturbed proximal algorithms presented in Algorithm \ref{algo:PPA} for obtaining an $\epsilon$-approximate local minimum. For the ease of presentation, we define in Table \ref{table:update_map} some model functions for the three proximal algorithms. By making assumptions on the function model $f_x(\cdot)$, we can analyze all proximal algorithms in a unified manner.

\begin{table}[t]
  \centering
  \begin{tabular}{lll}
    \toprule
    Algorithm& Objective  & Model function $f_x(y)$ \\
    \midrule
    Prox-point  &   $f(x)$ & $ f(y)$ \\
    Prox-gradient  & $ g(x)+r(x)$  & $ g(x)+\langle \nabla g(x),y-x\rangle+{m}(y)$  \\
    Prox-linear  & $ h(F(x))+r(x)$ & $  h(F(x)+\nabla F(x)(y-x))+{m}(y)$  \\
    \bottomrule
  \end{tabular}
   \caption{The model functions $f_x(y)$ for the three problems \eqref{ppa-prob}, \eqref{pgm-prob} and \eqref{plm-prob}.} %three algorithms with the update $x_{t+1} = S_\eta(x_t)$; we assume $h$ is convex, $f$ and $r$ are weakly convex, and both $g$ and $F$ are smooth.}
   \label{table:update_map}
\end{table}

The following assumption is from \cite{davis2019active} and is assumed throughout this section.

\begin{assumption}[\cite{davis2019active}]\label{assump:function}
\begin{enumerate}
\item Function $f(\cdot)$ is nonsmooth weakly convex and admits a $C^2$ (composite) active manifold around all critical points.
\item For all $x\in \RR^d$, there exists a constant $\beta >0$ such that the function model $f_x: \RR^d \to \RR\cup\{\infty\}$ satisfies
    \[\norm{f(y) - f_x(y)} \le \frac{\beta}{2}\norm{y-x}^2, \quad \text{ for all } y \in \RR^d.\]
\item The function model $f_x: \RR^d \to \RR\cup\{\infty\}$ itself is $\mu$-weakly convex.
\end{enumerate}
\end{assumption}

When the objective function $f(\cdot)$ admits a $C^2$ active manifold around critical points, the operators $S(\cdot)$ defined in \eqref{ppa-alg}, \eqref{pgm-alg} and \eqref{plm-alg} are $C^1$ smooth. Here, we make one additional assumption that $\nabla S$ is Lipschitz continuous. %also assume that the operator has local $\eta\rho$-Lipschitz gradient.

\begin{assumption}\label{assump:LipschitzS}
There exists a constant $\rho > 0$ such that the following inequality holds in the neighborhood $\U$ of all critical points:
\[\norm{\nabla S(x) - \nabla S(y)} \le \eta\rho\norm{x-y}, \quad \forall x, y \in \U,\]
where $\eta=\lambda$ in \eqref{ppa-alg}.
\end{assumption}
Note that when $f$ is a smooth function, this assumption reduces to the Hessian Lipschitz assumption.

Now we are ready to present our main iteration complexity result. %convergence theorem. We first choose following parameters:

%Finally, we can give our unified theorem for all three proximal operators.
\begin{theorem}\label{thm:main}
Suppose function $f(\cdot)$ and its model function $f_x(\cdot)$ satisfy Assumptions \ref{assump:function} and \ref{assump:LipschitzS}. For any $\epsilon \in (0, L^2/\rho)$ and $\delta\in(0,1)$, we set the following parameters:
\begin{equation}\label{eq:para}
\bad
&L > \frac{7}{2}\beta +  3\mu , \quad \lambda = \left(\frac{1}{2}L + \frac{1}{4}(\beta + 2\mu)\right)^{-1},\\
& \theta_1 = \frac{( L - \lambda^{-1}+\beta)\lambda^2L^2}{(L - \lambda^{-1}  - \beta )},\quad \theta_2 = \frac{(\lambda^{-1} - \beta - \mu )L\lambda}{(\lambda^{-1} + L - \beta - 2\mu)} ,\\
&\eta = \frac{1}{L}, \quad r =\frac{\theta_2}{\theta_1}\cdot \frac{\epsilon}{400 \iota^3}, \quad \utime = \frac{L}{ \sqrt{\rho\epsilon}} \cdot\iota,
\quad \ufun =\frac{\theta_2}{\theta_1}\cdot \frac{1}{50 \iota^3} \sqrt{\frac{\epsilon^3}{\rho}}, \quad
\uspace = \frac{1}{4\iota} \sqrt{\frac{\epsilon}{\rho}},
\ead
\end{equation}
where
\be\label{main-thm-def-iota}\iota = c \cdot \log( d L(f_\lambda(x_0) - f_\lambda^\star)/(\rho\epsilon\delta)),\ee
and $c$ is a sufficiently large absolute constant. Moreover, we set the number of iterations as
\be\label{main-thm-def-T}
T = 4 \max\left\{\frac{(f_\lambda(x_0) - f_\lambda^*)}{\ufun / (2\utime)}, \frac{(f_\lambda(x_0) - f_\lambda^*)}{\theta_2 \eta\epsilon^2/2}\right\}  = O\left(\frac{\theta_1L(f_\lambda(x_0) - f_\lambda^*)}{\theta_2\epsilon^2} \cdot (\log d)^4\right).
\ee
Then the Perturbed Proximal algorithms (Algorithm~\ref{algo:PPA}) for all three proximal operators have the property that with probability at least $1-\delta$, at least one half of its iterations will be $\epsilon$-approximate local minimum (Definition~\ref{defn:localmin_general}).
\end{theorem}

\subsection{Proof Sketch}
In \cite{jin2019nonconvex}, Jin \etal analyzed the iteration complexity of the perturbed gradient descent method for minimizing smooth functions. Our proof here largely follows the one in \cite{jin2019nonconvex}. The new ingredient is to extend the results in the smooth case to the nonsmooth case by utilizing properties of the Moreau envelope and proximal operators.

Following the main proof ideas of \cite{jin2019nonconvex}, we give the following two lemmas which guarantee sufficient decrease of the Moreau envelope.

\begin{lemma}[Sufficient Decrease of Moreau envelope]\label{lemma:prox_decrease}
%By setting $L > \max\{2(\beta + \mu), \frac{5}{2}\beta + \mu \}$ and $\lambda = (\frac{1}{2}L + \frac{1}{4}(\beta + 2\mu))^{-1}$,
With parameters chosen as in \eqref{eq:para}, we have
\be\label{lemma-4.7-eq-1}
\norm{x_{t+1} -x_t}^2 \le \theta_1\eta^2\norm{\nabla f_\lambda(x_t)}^2,
\ee
where $\theta_1 = \frac{( L - \lambda^{-1} +\beta)\lambda^2L^2}{(L - \lambda^{-1}  - \beta )}$.
Moreover, we have
\be \label{eq:suff_decrease}
\bad
f_\lambda(x_{t+1}) \le f_\lambda(x_{t}) -\frac{\theta_2}{2L} \norm{\nabla f_\lambda(x_t)}^2,
\ead
\ee
where $\theta_2 = \frac{(\lambda^{-1} - \beta - \mu )L\lambda}{(\lambda^{-1} + L - \beta - 2\mu)} > 0.$
\end{lemma}

\begin{lemma}[Escaping saddle point]\label{lemma:ESP}
Suppose Assumptions \ref{assump:function} and \ref{assump:LipschitzS} hold for $f(\cdot)$, and $\tilde{x}$ satisfies $\norm{\nabla f_\lambda(\tilde{x})} < \epsilon$ , and $\lambda_{\max}(\nabla S(\tilde{x})) > 1 + \eta\sqrt{\rho\epsilon}$ . Let $x_0 = \tilde{x}  + \eta\xi \quad(\xi \sim \text{Uniform}(\BB_0(r)))$. Run the update $x_{t+1} = S(x_t)$ as in \eqref{ppa-alg}, \eqref{pgm-alg} and \eqref{plm-alg} starting from $x_0$. We have
\[
\PP(f_\lambda(x_{\utime}) - f_\lambda(\tilde{x}) \le -\ufun/2) \ge 1 - \frac{L\sqrt{d}}{\sqrt{\rho\epsilon}} \cdot \iota^2 2^{8-\iota}, %\frac{\tau\sqrt{d}}{\rho\epsilon}\cdot \iota 2^{8-\iota},
\]
where $x_{\utime}$ is the $\utime^{th}$ proximal iterate starting from $x_0$.
\end{lemma}

\begin{remark}
Inequality \eqref{eq:suff_decrease} guarantees the sufficient decrease of the Moreau envelope when $\norm{\nabla f_\lambda(x)}$ is large, which is deterministic. On the other hand, Lemma \ref{lemma:ESP} shows that, after random perturbation and running the proximal updates for $\utime$ iterations, the Moreau envelope has sufficient decrease with high probability when stuck in a saddle point. Since the Moreau envelope can be decreased by at most $f_\lambda(x_0) - f_\lambda^*$, it is straightforward to give a bound for those iterations that are not approximate local minimums.
\end{remark}

\section{Conclusion {and Future Directions}}
In this paper, we have presented a novel definition of $\epsilon$-approximate local minimum for nonsmooth weakly convex functions based on the recent work on active strict saddles \cite{davis2019active}. Following the idea of perturbed gradient descent method \cite{jin2019nonconvex}, we have proposed a unified algorithm framework for perturbed proximal algorithms, namely, perturbed proximal point, perturbed proximal gradient and  perturbed proximal linear algorithms for obtaining an $\epsilon$-approximate local minimum. {We have proved that the proposed perturbed proximal algorithms achieve a gradient computational cost of $\tilde{O}(\epsilon^{-2})$ for finding an $\epsilon$-approximate local minimum. With this novel characterization of local minimum for nonsmooth functions in hand, we give some possible future directions:}
 \begin{enumerate}
\item {Following Fang \etal \cite{fang2019sharp} for smooth problems, one can attempt to extend our algorithms to their stochastic versions for nonsmooth problems. It is interesting to analyze the complexity of these stochastic proximal algorithms for obtaining an $\epsilon$-approximate local minimum of nonsmooth weakly convex functions.} %We expect that the stochastic proximal algorithms can achieve a rate of $\tilde{O}(\epsilon^{-3.5})$ for finding an $\epsilon$-approximate local minimum of nonsmooth weakly convex functions.}
\item Recently, Fang \etal \cite{fang2018spider} combined the negative curvature search with variance reduction technique for stochastic algorithms to find a local minimum and reach a remarkable rate of $\tilde{O}(\epsilon^{-3})$. It will be interesting to adopt this technique for escaping saddle points for nonsmooth problems.% and see if the stochastic proximal algorithms can reach this state-of-art $\tilde{O}(\epsilon^{-3})$ stochastic gradient computational cost.}
\item {Another line of research for finding local minimum of smooth functions is to apply second order method \cite{nesterov2006cubic}. It will be also interesting to investigate if second order proximal methods can find local minimum with better convergence rate. }
\end{enumerate}

\newpage
\bibliography{escapesaddle}
\bibliographystyle{plain}

\newpage
\appendix
\section{Preliminaries on nonsmooth optimization}

We review some key definitions and tools for nonsmooth nonconvex optimization.
%\textbf{Subdifferential and Subderivatives:} To deal with nonsmoothness, we first introduce the standard definition of Subdifferential and Subderitives. Further details can be found in \cite{rockafellar2009variational}.
\begin{definition}[Subdifferential and Subderivative \cite{rockafellar2009variational}]\label{defn:sub_diff}
Consider a function $f: \RR^d \to \RR\cup\{\infty\}$. The subdifferential of $f$ at $x$, denoted as $\partial f(x)$, consists of all vectors $v$ satisfying
\be
\bad
f(y) \ge f(x) + \langle v, x - y \rangle + o(\norm{x - y}) \quad \text{ as } y \to x.
\ead
\ee
The subderivative of function $f$ at $x$ in the direction $\bar{u} \in \RR^d$ is
\be
\bad
df(x)(\bar{u}) := \liminf_{t \to 0, u \to \bar{u}} \frac{f(x + tu) - f(x)}{t}.
\ead
\ee
The parabolic subderivative of $f$ at $x$ for $\bar{u} \in \text{dom } df (x)$ with respect to $\bar{w}$ is
\be
\bad
d^2f(x)(\bar{u}|\bar{w}) := \liminf_{t \to 0, w \to \bar{w}} \frac{f(x + t\bar{u} + \frac{1}{2}t^2w) - f(x) - df(x)(\bar{u})}{\frac{1}{2}t^2}.
\ead
\ee
\end{definition}

\section{Proofs of the Main Results}

\subsection{Proof of Lemma \ref{lem_prop_moreau}}
\begin{proof}
Claims 1 and 2 follow from \cite[Lemma 2.5]{davis2019active}. Now we prove Claim 3.
By \eqref{eqn:grad_moreau} and the observation that function $y \mapsto f(y) + \frac{1}{2\lambda}\norm{y-x}^2$, we know that all critical points of $f_\lambda(\cdot)$ and $f(\cdot)$ coincide. Moreover, if $\bar{x}$ is a local minimum of $f_\lambda(\cdot)$, then \eqref{eqn:grad_moreau} implies that 
$\bar{x}=\prox_{\lambda f}(\bar{x})$. Therefore, 
\[
f_\lambda(\bar{x}) = \min_{y \in \RR^d} f(y) + \frac{1}{2\lambda}\norm{y-\bar{x}}^2 = f(\prox_{\lambda f}(\bar{x})) + \frac{1}{2\lambda}\norm{\prox_{\lambda f}(\bar{x})-\bar{x}}^2 = f(\bar{x}).
\]
Hence, if $\bar{x}$ is a local minimum of $f_\lambda(\cdot)$, then there exists some constant $\delta > 0 $ such that for any ${x}$ in the neighborhood of $\bar{x}$: $\U = \{x |0 < \norm{x - \bar{x} } \le \delta\},$ we have $f_\lambda({x}) \ge f_\lambda(\bar{x})$. This further implies that
\[f(x)\geq f_\lambda(x) \geq f_\lambda(\bar{x}) = f(\bar{x}),\]
where the first inequality is due to the fact that $f_\lambda(\cdot)$ is always a lower bound to $f(\cdot)$.
%Then $f(\bar{x}) \ge f_\lambda(\bar{x}) \ge f_\lambda(x) \red{=} f(x)$ for $\bar{x} \in \U$. Therefore, $x$ is a local minimum of $f(\cdot)$
%We further notice that the Moreau envelope  $f_\lambda(\cdot)$ is a lower bound of $f(\cdot)$:
%\[
%f_\lambda(x) = \min_{y \in \RR^d} f(y) + \frac{1}{2\lambda}\norm{y-x}^2 \le f(x).
%\]
%If $x$ is a local minimum of $f_\lambda(\cdot)$, then there exist some $\delta > 0 $ such that for any $\bar{x}$ in the neighborhood $\U = \{\bar{x} |0 < \norm{\bar{x} - x} \le \delta\},$ we have $f_\lambda(\bar{x}) \ge f_\lambda(x)$. Then $f(\bar{x}) \ge f_\lambda(\bar{x}) \ge f_\lambda(x) \red{=} f(x)$ for $\bar{x} \in \U$. Therefore, $x$ is a local minimum of $f(\cdot)$. \red{how did you get this red equation?}
\end{proof} 

\subsection{Proof of Theorem \ref{thm:main}}

\begin{proof}
The proof largely follows \cite{jin2019nonconvex}, which proved similar results for the smooth case. We need to extend the results to nonsmooth case.
We use the Moreau envelope to measure the first-order stationarity. We show that the three perturbed proximal algorithms make sufficient decrease on the norm of the Moreau envelope with high probability if the current iterate is not a local minimum. For $x \in \RR^d$, we have three possible cases:
\begin{enumerate}
\item When $\norm{\nabla f_\lambda(x)} > \epsilon$, we show that the regular proximal update can guarantee a sufficient decrease on the Moreau envelope $f_\lambda$.
\item When  $\norm{\nabla f_\lambda(x)} \le \epsilon$ and $\lambda_{\max}(\nabla S(x)) > 1 + \eta\sqrt{\rho\epsilon}$, $x$ is an unstable fixed point, which corresponds to a strict saddle point or a local maximum. In this case, we show that running a perturbed proximal update will guarantee a sufficient decrease in the next $\utime$ iterations with high probability.
\item When  $\norm{\nabla f_\lambda(x)} \le \epsilon$ and $\lambda_{\max}(\nabla S(x)) \le 1 + \eta\sqrt{\rho\epsilon}$, we have reached an approximate local minimum.
\end{enumerate}
The case (i) can be proved by Lemma \ref{lemma:prox_decrease}. When $\norm{\nabla f_\lambda(x_t)} > \epsilon$, from \eqref{eq:suff_decrease} we have
\[
f_\lambda(x_{t+1}) \le f_\lambda(x_{t}) -\frac{\theta_2}{2L} \norm{\nabla f_\lambda(x_t)}^2 \le f_\lambda(x_{t}) -\frac{\theta_2\epsilon^2}{2L}.
\]
That is, there is a sufficient reduction of $f_\lambda(x_t)$ after one proximal update. Note that $f_\lambda$ can be decreased at most $f_\lambda(x_0) - f_\lambda^*$. Therefore, the number of iterations in which case (i) happens is at most $\frac{(f_\lambda(x_0) - f_\lambda^*)}{\theta_2 \epsilon^2/(2L)} \leq T/4$, where $T$ is defined in \eqref{main-thm-def-T}. %, where $f_\lambda^*$ is the Moreau Envelop at the global optimum. Since each regular proximal update will reduce the Moreau Envelop by at least $\frac{\theta_2\epsilon^2}{2L}$ when $\norm{\nabla f_\lambda(x)} > \epsilon$, we can run those regular proximal updates for at most $\frac{(f_\lambda(x_0) - f_\lambda^*)}{\theta_1 \epsilon^2/(2L)}$ times for case 1.\\

For case (ii), from Lemma \ref{lemma:ESP} we know that with high probability, adding a stochastic perturbation can guarantee a decrease of the Moreau envelope by at least $\ufun /2$ after $\utime$ iterations. %Note that we allow $\iota$ to be large enough so that the probability $1 - \frac{\tau\sqrt{d}}{\rho\epsilon}\cdot \iota 2^{8-\iota}$ is close to $1$.
%In this case, we have an average decrease of $\ufun / (2\utime)$ for each step.
Therefore, the number of iterations in which case (ii) happens is at most $\frac{(f_\lambda(x_0) - f_\lambda^*)}{\ufun / (2\utime)}$. %Similar with the first case, we can have at most $\frac{(f_\lambda(x_0) - f_\lambda^*)}{\ufun / (2\utime)}$ perturbed steps.
Choose a large enough absolute constant $c$ such that:
\[
(T L\sqrt{d} /\sqrt{\rho\epsilon} ) \cdot\iota^2  2^{8-\iota} \le \delta,
\]
where $T$ is defined in \eqref{main-thm-def-T}, and $\iota$ is defined in \eqref{main-thm-def-iota}.
We see that with probability $1 - \delta$, we add the stochastic perturbation in at most  $\frac{(f_\lambda(x_0) - f_\lambda^*)\utime}{\ufun / 2} \leq T/4$ iterations. That is, the number of iterations in which case (ii) happens is at most $T/4$.

Finally, we conclude that case (iii) happens in at least $T/2$ iterations. That is, at least $T/2$ of the iterates in the first $T$ iterations must be $\epsilon$-approximate local minimum. As defined in \eqref{main-thm-def-T}, the total number of iterations $T$ relies on the problem dimension $d$ polylogarithmically.
\end{proof}

\subsection{Proof of Lemma \ref{lemma:prox_decrease}}

To prove Lemma \ref{lemma:prox_decrease}, we need to prove the following lemma first.
%\begin{proof}
%We first give the proof of Lemma \ref{lemma:prox_decrease}, which highly relies on the weakly convexity of the objective function. The main idea is to derive a three-point inequality(the current point $x_t$, the proximal mapping to the current point $\bar{x}_t = \prox_{\lambda f}(x_t)$ and the next point $x_{t+1}$) for all proximal update. We have the following Lemma:
\begin{lemma}\label{lemma:triangle}
Suppose Assumption \ref{assump:function} holds and parameters $L$ and $\lambda$ are chosen as in \eqref{eq:para}, we have the following three point inequality:
\begin{align}\label{eq:triangle}
    & (L - \lambda^{-1}-\beta)\norm{x_{t+1} - x_t}^2 \\
\le & ( L - \lambda^{-1}+ \beta)\norm{\bar{x}_{t} - x_t}^2 - (L + \lambda^{-1} - \beta - 2\mu) \norm{\bar{x}_t - x_{t+1}}^2,\nonumber
\end{align}
where $\bar{x}_t := \prox_{\lambda f}(x_t)$.
\end{lemma}

\begin{proof}
Under the Assumption \ref{assump:function}, we know that function $f(\cdot)$ is $(\beta + \mu)$-weakly convex. With parameter $\eta$ chosen in \eqref{eq:para} and the model functions $f_x$ defined in Table \ref{table:update_map}, proximal updates in \eqref{ppa-alg}, \eqref{pgm-alg} and \eqref{plm-alg} can be written as:
\[
x_{t+1} = \argmin_{y\in \RR^d} \left\{f_{x_t}(y) + \frac{L}{2}\norm{y-x}^2\right\}.
\]
%We first utilize the strongly convexity of the problem of solving the Moreau envelope.
From \eqref{eq:para} we have $\lambda^{-1} = \frac{1}{2}L + \frac{1}{4}(\beta + 2\mu) > \beta + \mu$, and therefore function $y \mapsto f(y) + \frac{\lambda^{-1}}{2}\norm{y - x_t}^2$ is $(\lambda^{-1} - \beta- \mu)$-strongly convex. We have %With the proximal mapping denoted as $\bar{x}_t = \prox_{\lambda f}(x_t) = \argmin_{y \in \RR^d} f(y) + \frac{\lambda^{-1}}{2}\norm{y - x_t}^2$, we have
\be \label{strongly_moreau}
\frac{\lambda^{-1} - \beta - \mu}{2} \norm{\bar{x}_t - x_{t+1}}^2 \le \left(f(x_{t+1}) + \frac{\lambda^{-1}}{2}\norm{x_{t+1} - x_t}^2\right) -\left (f(\bar{x}_{t}) + \frac{\lambda^{-1}}{2}\norm{\bar{x}_{t} - x_t}^2\right).
\ee
Using part 2 of Assumption \ref{assump:function}, we have
\be \label{double_side}
f(x_{t+1}) \le f_{x_t}(x_{t+1}) + \frac{\beta}{2} \norm{x_{t+1} - x_t}^2, \text{ and }  -f(\bar{x}_t) \le -f_{x_t}(\bar{x}_t) + \frac{\beta}{2} \norm{\bar{x}_t - x_t}^2.
\ee
Combining \eqref{double_side} and \eqref{strongly_moreau} yields
\be \label{eq:strongly1}
\frac{\lambda^{-1} - \beta - \mu}{2} \norm{\bar{x}_t - x_{t+1}}^2 \le f_{x_t}(x_{t+1}) -  f_{x_t}(\bar{x}_{t}) +  \frac{\lambda^{-1}+\beta}{2}\norm{x_{t+1} - x_t}^2 -  \frac{\lambda^{-1}- \beta}{2}\norm{\bar{x}_{t} - x_t}^2.
\ee
On the other hand, we have $L > \mu$ and function $y \mapsto f_{x_t}(y) + \frac{L}{2}\norm{y - x_t}^2$ is $(L - \mu)$-strongly convex and $x_{t+1}$ is its minimizer. So we have
\be\label{eq:strongly2}
\bad
 f_{x_t}(x_{t+1}) -  f_{x_t}(\bar{x}_{t}) \le  \frac{L}{2}\norm{\bar{x}_{t} - x_t}^2 -  \frac{L}{2}\norm{x_{t+1} - x_t}^2 -  \frac{L - \mu}{2}\norm{x_{t+1} - \bar{x}_t}^2.
\ead
\ee
Combining \eqref{eq:strongly1} and \eqref{eq:strongly2}, we have
%\be \label{eq:threepoint}
%\bad
\begin{align*}
\frac{\lambda^{-1} - \beta - \mu}{2} \norm{\bar{x}_t - x_{t+1}}^2 &\le  \frac{L}{2}\norm{\bar{x}_{t} - x_t}^2 -  \frac{L}{2}\norm{x_{t+1} - x_t}^2 -  \frac{L - \mu}{2}\norm{x_{t+1} - \bar{x}_t}^2\\
&  +  \frac{\lambda^{-1}+\beta}{2}\norm{x_{t+1} - x_t}^2 -  \frac{\lambda^{-1}- \beta}{2}\norm{\bar{x}_{t} - x_t}^2,
\end{align*}
%\ead
%\ee
which leads to \eqref{eq:triangle} and completes the proof of Lemma \ref{lemma:triangle}.
%\be \label{eq:triangle}
%\bad
% \frac{L - \lambda^{-1}-\beta}{2}\norm{x_{t+1} - x_t}^2 \le  \frac{ L - \lambda^{-1}+ \beta}{2}\norm{\bar{x}_{t} - x_t}^2 - \frac{L + \lambda^{-1} - \beta - 2\mu}{2} \norm{\bar{x}_t - x_{t+1}}^2.
%\ead
%\ee
%The proof of Lemma \ref{lemma:triangle} is completed.
\end{proof}

We now prove Lemma \ref{lemma:prox_decrease}.

\begin{proof}
We note that with $L > \frac{5}{2}\beta + \mu$, and $\lambda^{-1} = \frac{1}{2}L + \frac{1}{4}(\beta + 2\mu)$, we have the coefficients $L - \lambda^{-1}-\beta > 0$,  $L - \lambda^{-1}+ \beta > 0$, and $L + \lambda^{-1} - \beta - 2\mu > 0$. Thus, \eqref{eq:triangle} implies
\[
\norm{x_{t+1} - x_t}^2 \le  \frac{ L - \lambda^{-1}+ \beta}{L - \lambda^{-1}-\beta}\norm{\bar{x}_{t} - x_t}^2 = \frac{ L - \lambda^{-1}+ \beta}{L - \lambda^{-1}-\beta}\lambda^2\norm{\nabla f_\lambda(x_t)}^2 = \theta_1 \eta^2\norm{\nabla f_\lambda(x_t)}^2,
\]
where $\theta_1 = \frac{ (L - \lambda^{-1}+ \beta)\lambda^2L^2}{L - \lambda^{-1}-\beta} > 0$ and the first equality is due to \eqref{eqn:grad_moreau}. %from which we have an equality $\nabla f_\lambda(x) = \frac{x - \prox_{\lambda f}(x)}{\lambda}$ hold.
This completes the proof of \eqref{lemma-4.7-eq-1}.

Next, we prove \eqref{eq:suff_decrease}. Note that
\be \label{moreau_decre}
f_\lambda(x_{t+1}) = f(\bar{x}_{t+1}) + \frac{\lambda^{-1}}{2}\norm{\bar{x}_{t+1} - x_{t+1}}^2\le f(\bar{x}_{t}) + \frac{\lambda^{-1}}{2}\norm{\bar{x}_{t} - x_{t+1}}^2,
\ee
where the inequality is due to $\bar{x}_{t+1}:=\prox_{\lambda f}(x_{t+1})$. % is the unique minimizer of $y \mapsto f(y) + \frac{\lambda^{-1}}{2}\norm{y - x_{t+1}}^2$.
Combining \eqref{moreau_decre} and \eqref{eq:triangle} yields
\be \label{eq:triangle2}
\bad
\norm{\bar{x}_t - x_{t+1}}^2 &\le  \frac{ L - \lambda^{-1}+ \beta}{L + \lambda^{-1} - \beta - 2\mu}\norm{\bar{x}_{t} - x_t}^2 -  \frac{L - \lambda^{-1}-\beta}{L + \lambda^{-1} - \beta - 2\mu}\norm{x_{t+1} - x_t}^2\\
& = \norm{\bar{x}_{t} - x_t}^2 - \frac{2(\lambda^{-1}- \beta - \mu)}{L + \lambda^{-1} - \beta - 2\mu}\norm{\bar{x}_{t} - x_t}^2  -  \frac{(L - \lambda^{-1}-\beta) \lambda^2}{L + \lambda^{-1} - \beta - 2\mu}\norm{\nabla f_\lambda(x_t)}^2\\
& \le \norm{\bar{x}_{t} - x_t}^2 -  \frac{(L - \lambda^{-1}-\beta) \lambda^2}{L + \lambda^{-1} - \beta - 2\mu}\norm{\nabla f_\lambda(x_t)}^2.
\ead
\ee
%The first inequality is by rearranging Lemma \ref{lemma:triangle} and the fact that $L + \lambda^{-1} - \beta - 2\mu > 0$. The last inequality is due to $\lambda^{-1}- \beta - \mu > 0.$\\
Substituting \eqref{eq:triangle2} into \eqref{moreau_decre} gives
\begin{align*}
f_\lambda(x_{t+1}) &\le f(\bar{x}_{t}) + \frac{\lambda^{-1}}{2}\norm{\bar{x}_{t} - x_{t}}^2 - \frac{\lambda^{-1}}{2}  \frac{(L - \lambda^{-1}-\beta) \lambda^2}{L + \lambda^{-1} - \beta - 2\mu}\norm{\nabla f_\lambda(x_t)}^2\\
& = f_\lambda(x_{t}) -  \frac{(L - \lambda^{-1}-\beta) \lambda}{2(L + \lambda^{-1} - \beta - 2\mu)}\norm{\nabla f_\lambda(x_t)}^2\\
& = f_\lambda(x_{t}) -  \frac{\theta_2}{2L}\norm{\nabla f_\lambda(x_t)}^2,
\end{align*}
where $\theta_2 = \frac{(L - \lambda^{-1}-\beta) \lambda L}{(L + \lambda^{-1} - \beta - 2\mu)} > 0$, which completes the proof.
\end{proof}

\subsection{Proof of Lemma \ref{lemma:ESP}}

The main idea of Lemma \ref{lemma:ESP} is again from \cite{jin2019nonconvex} for the perturbed gradient descent method, and the proof also largely follows from \cite{jin2019nonconvex}. The ``improve or localize'' idea shows that if the Moreau envelope does not decrease a lot, then the sequence $x_t$ must stay in a neighborhood of $x_0$. Utilizing this idea, we show that the escaping area of the perturbation ball is small. We summarize the whole idea in the following two lemmas.

The first Lemma introduces the idea of improve or localize:
\begin{lemma}[Improve or Localize]\label{lemma:IOL}
Under the setting of Lemma~\ref{lemma:prox_decrease}, for any $t \ge \tau > 0$:
\[
\norm{x_\tau - x_0} \le \sqrt{\frac{\theta_1}{\theta_2}\cdot 2\eta t(f_\lambda(x_0) - f_\lambda(x_t))}.
\]
\end{lemma}
\begin{proof} We have
\begin{align*}
\norm{x_\tau - x_0} &\le \sum_{\tau = 1}^{t}\norm{x_\tau - x_{\tau-1}} \le \left[t  \sum_{\tau = 1}^{t}\norm{x_\tau - x_{\tau-1}}^2\right]^{1/2}\\
& \le \left[\theta_1 \eta^2 t  \sum_{\tau = 1}^{t}\norm{\nabla f_\lambda(x_{\tau-1})}^2\right]^{1/2} \le \sqrt{\frac{\theta_1}{\theta_2}\cdot 2\eta t (f_\lambda(x_0) - f_\lambda(x_t)) },
\end{align*}
where the first inequality comes from the triangle inequality, the second inequality is due to $ (\sum_{\tau = 1}^{t}\norm{a_\tau})^2 \le t  \sum_{\tau = 1}^{t}\norm{a_\tau}^2$, and the last two inequalities are from Lemma \ref{lemma:prox_decrease}.
\end{proof}

The second lemma shows that when two coupling sequences initiated along the maximal eigen-direction with a small distance, at least one of them can make sufficient decrease on the Moreau envelope.

\begin{lemma}[Coupling sequence]\label{lemma:CS}
Suppose $f$ satisfies Assumptions  \ref{assump:function} and \ref{assump:LipschitzS}, $\tilde{x}$ satisfies $\norm{\nabla f_\lambda(\tilde{x})} < \epsilon$, and $\lambda_{\max}(\nabla S(\tilde{x})) > 1 + \eta\sqrt{\rho\epsilon}$. Let $\{x_t\}, \{x'_t\}$ be two proximal update sequences that satisfy: (i) $\max\{\norm{x_0 - \tilde{x}},\norm{x'_0 - \tilde{x}} \} \le \eta r$; and (ii) $x_0 - x'_0 = \eta r_0 \textbf{e}_1$, where $\textbf{e}_1$ is the eigenvector corresponding to the largest eigenvalue of $\nabla S(\tilde{x})$ and $r_0 > \omega := 2^{2-\iota}L\uspace$. Then:
$$\min\{ f_\lambda(x_{\utime}) - f_\lambda(x_{0}), f_\lambda(x'_{\utime}) - f_\lambda(x'_{0}) \} \le - \ufun.$$
\end{lemma}

\begin{proof}
We prove it by contradiction. First assume that $\min\{ f_\lambda(x_{\utime}) - f_\lambda(x_{0}), f_\lambda(x'_{\utime}) - f_\lambda(x'_{0}) \} > - \ufun.$ By Lemma \ref{lemma:IOL}, we have
\be \label{eq:local}
\bad
\max\{ \norm{x_t - \tilde{x}},  \norm{x'_t - \tilde{x}} \} &\le \max\{ \norm{x_t - x_0},  \norm{x'_t - x'_0} \} + \{ \norm{x_0 - \tilde{x}},  \norm{x'_0 - \tilde{x}} \}\\
& \le \sqrt{\frac{\theta_1}{\theta_2}\cdot 2\eta \utime \ufun} + \eta r \le \uspace.
\ead
\ee
The last inequality holds due to the choices of the parameters in \eqref{eq:para} and that we can choose $\iota$ to be sufficiently large so that $\eta r$ is small enough compared to the first term. On the other hand, we write the difference $\hat{x}_t := x_t - x'_t$ and update it as:
\be
\bad
\hat{x}_{t+1} &= S(x_t) - S(x'_t) = \mathcal{H}\hat{x}_t + \Delta_t\hat{x}_t\\
& = \underbrace{\mathcal{H}^{t+1}\hat{x}_0}_{p(t+1)}  + \underbrace{\sum_{\tau = 0}^{t}\mathcal{H}^{t-\tau} \Delta_\tau\hat{x}_\tau}_{q(t+1)},
\ead
\ee
where $ \mathcal{H} = \nabla S(\tilde{x})$ and $\Delta_t = \int_0^1 [\nabla S(x'_t + \theta (x_t - x'_t)) - \mathcal{H}]d\theta$. Next we show that $\norm{p(t)}$ is the dominant term by showing
\be\label{eq:q_dorm}
\bad
\norm{q(t)} \le \norm{p(t)} /2,  \qquad t \in [\utime].
\ead
\ee
By induction, $\norm{q(0)} = 0 \le \norm{\hat{x}_0}/2 = \norm{p(0)}/2$. Therefore, \eqref{eq:q_dorm} holds when $t = 0$. We assume it still holds for some $t > 0$. Denote $\lambda_{\max}(\nabla S(\tilde{x})) = 1 + \eta\gamma$, so that $\gamma > \sqrt{\rho\epsilon}$. Since $\hat{x}_0$ lies in the direction of maximum eigenvector of $\nabla S(\tilde{x})$, for any $\tau \le t$, we have:
\be
\bad
\norm{\hat{x}_\tau}  \le \norm{p(\tau)} + \norm{q(\tau)} \le 2\norm{p(\tau)} = 2\norm{\mathcal{H}^{\tau}\hat{x}_0} = 2(1+\eta\gamma)^\tau \eta r_0.
\ead
\ee
By assumption \ref{assump:LipschitzS}, we have
{
\be
\bad
\norm{\Delta_t} &\le \int_0^1 \norm{\nabla S(x'_t + \theta (x_t - x'_t)) - \nabla S(\tilde{x})}d\theta\\
& \le \eta\rho  \int_0^1\norm{\theta(x_t - \tilde{x}) + (1-\theta) (x'_t - \tilde{x})}\\
& \le \eta\rho \max\{\norm{x_t - \tilde{x}}, \norm{x'_t - \tilde{x}} \} \\
& \le \eta\rho \uspace,
\ead
\ee}
hence
\be
\bad
\norm{q(t+1)} & = \norm{\sum_{\tau = 0}^{t}\mathcal{H}^{t-\tau} \Delta_\tau \hat{x}_\tau }
\le \eta\rho\uspace \sum_{\tau = 0}^{t} \norm{\mathcal{H}^{t-\tau}}\norm{\hat{x}_\tau} \\
&\le 2\eta\rho\uspace \sum_{\tau = 0}^{t} (1+\eta\gamma)^t \eta r_0
\le 2\eta\rho\uspace \utime (1+\eta\gamma)^t \eta r_0\\
&\le 2\eta\rho\uspace \utime \norm{p(t+1)},
\ead
\ee
and we have $2\eta\rho\uspace \utime \le 1/2$, which completes the proof.\\
Finally, we have
\be
\bad
\max\{\norm{x_t - x_0},  \norm{x'_t - x'_0} \} &\ge \frac{1}{2}\norm{\hat{x}(\utime)} \ge \frac{1}{2}[\norm{p(\utime)} - \norm{q(\utime)}] \ge \frac{1}{4} \norm{p(\utime)}\\
& = \frac{(1 + \eta\gamma)^{\utime}\eta r_0}{4} \ge 2^{\iota - 2}\eta r_0 > \uspace,
\ead
\ee
which contradicts with \eqref{eq:local}. So we complete the proof.
\end{proof}

Now we prove Lemma \ref{lemma:ESP}.
\begin{proof}
We define the stuck region on the ball $\BB_{\tilde{x}}(\eta r)$ as
\be
\bad
\mathcal{X}_{\text{stuck}} := \{x \in \BB_{\tilde{x}}(\eta r) | \{x_t\} \text{is the iterate sequence with } x_0 = x, \text{and }f_\lambda(x_{\utime}) - f_\lambda(\tilde{x}) > -\ufun \},
\ead
\ee
and show that this region is a small part of the surface of the ball. According to Lemma \ref{lemma:CS} we know that the width of $\mathcal{X}_{\text{stuck}}$ along the $\textbf{e}_1$ direction is at most $\eta \omega$. Therefore,
\be
\bad
\Pr(x_0 \in \mathcal{X}_{\text{stuck}}) &= \frac{\text{Vol}( \mathcal{X}_{\text{stuck}})}{\text{Vol}(\BB^{d}_{\tilde{x}}(\eta r))}
\le \frac{\eta \omega \times \text{Vol}(\BB^{d-1}_0(\eta r))}{\text{Vol} (\BB^{d}_0(\eta r))}\\
&= \frac{\omega}{r\sqrt{\pi}}\frac{\Gamma(\frac{d}{2}+1)}{\Gamma(\frac{d}{2}+\frac{1}{2})}
\le \frac{\omega}{r} \cdot \sqrt{\frac{d}{\pi}} \le  \frac{L\sqrt{d}}{\sqrt{\rho\epsilon}} \cdot \iota^2 2^{8-\iota}.
\ead
\ee
From \eqref{eq:para} we have $\lambda^{-1} > \frac{\beta + \mu}{1 - \lambda(\beta + \mu)}$ and $L > \lambda^{-1}$. By Lemma \ref{lem_prop_moreau}, $f_\lambda$ is $L$-smooth. For those $\{x_0 \notin \mathcal{X}_{\text{stuck}}\}$, since $f_\lambda$ is $L$-smooth, we have:
 \be
\bad
f_\lambda(x_\utime) - f_\lambda(\tilde{x}) = [f_\lambda(x_\utime) - f_\lambda(x_0)] + [f_\lambda(x_0)- f_\lambda(\tilde{x})] \le -\ufun + \epsilon \eta r + \frac{L \eta^2 r^2}{2} \le -\ufun / 2,
\ead
\ee
where in the last inequality, we choose $\iota$ to be large enough so that the third term can be omitted. This completes the proof.
\end{proof}

\end{document}